\documentclass[10pt]{amsart}

\usepackage[utf8]{inputenc} 
\usepackage[T1]{fontenc} 
\usepackage{hyperref} 
\usepackage{url} 
\usepackage{booktabs} 
\usepackage{nicefrac}  
\usepackage{microtype} 
\usepackage{amsmath,amsthm,amsfonts,amssymb}
\usepackage{booktabs}
\usepackage{caption}
\usepackage{graphicx}
\usepackage[all]{nowidow}
\usepackage[]{tikz}
\usetikzlibrary{er,positioning,bayesnet}
\usepackage{multicol}
\usepackage[noend]{algpseudocode}
\usepackage{algorithm}
\usepackage{hyperref}
\usepackage{subfigure}
\usepackage[inline]{enumitem}

\newtheorem{lemma}{Lemma}
\newtheorem{prop}{Proposition}
\newtheorem{rem}{Remark}

\newcommand{\R}{\mathbb{R}}

\begin{document}
	
\title{Bounding The Number of Linear Regions in Local Area for Neural Networks with ReLU Activations}
\author{Rui Zhu}
\address{Department of Computer Science, Indiana University, Bloomington, IN 47408}
\email{zhu11@iu.edu}
\author{Bo Lin}
\address{School of Mathematics,	Georgia Institute of Technology, Atlanta, GA 30332}
\email{bo.lin@math.gatech.edu}
\author{Haixu Tang}
\address{Department of Computer Science, Indiana University, Bloomington, IN 47408}
\email{hatang@indiana.edu}

\begin{abstract}
The number of linear regions is one of the distinct properties of the neural networks using piecewise linear activation functions such as ReLU, comparing with those conventional ones using other activation functions. Previous studies showed this property reflected the {\em expressivity} of a neural network family (\cite{montufar2014number}); as a result, it can be used to characterize how the structural complexity of a neural network model affects the function it aims to compute. Nonetheless, it is challenging to directly compute the number of linear regions; therefore, many researchers focus on estimating the bounds (in particular the upper bound) of the number of linear regions for deep neural networks using ReLU. These methods, however, attempted to estimate the upper bound in the entire input space. The theoretical methods are still lacking to estimate the number of linear regions within a specific area of the input space, e.g., a sphere centered at a training data point such as an adversarial example or a backdoor trigger. In this paper, we present the first method to estimate the upper bound of the number of linear regions in any sphere in the input space of a given ReLU neural network. We implemented the method, and computed the bounds in deep neural networks using the piece-wise linear active function. Our experiments showed that, while training a neural network, the boundaries of the linear regions tend to move away from the training data points. In addition, we observe that the spheres centered at the training data points tend to contain more linear regions than any arbitrary points in the input space. To the best of our knowledge, this is the first study of bounding linear regions around a specific data point. We consider our research, with both theoretical proof and software implementation, as a first step toward the investigation of the structural complexity of deep neural networks in a specific input area.
\end{abstract}

\maketitle

\section{Introduction}
It is well recognized that deep neural networks (DNNs) have an impressive fitting ability. The structural maximum complexity grows exponentially with the increasing size (i.e., the number of layers and neurons) of DNNs. Because of this great property, some previous researches (\cite{sontag1998vc,harvey2017nearly}) explained why neural networks were able to approximate extremely complex functions with fewer parameters than other machine learning models. In particular, for DNNs using ReLU, the most popular activation function, their structural complexity (or {\em expressivity}) are often measured by the number of linear regions(Fig \ref{fig:illustration}(a)). However, in practice, the number of linear regions is estimated by using relatively loose upper bounds: as shown recently, Deep ReLU networks had surprisingly fewer linear regions than the maximum that they can reach (\cite{hanin2019deep}). Besides, in some cases, it is conceivable that an additional neuron may lead to no change of the model's expressivity at all. On the other hand, counting the exact number of linear regions was shown to be very hard  (\cite{lin2017linear,serra2017bounding}). As a result, many recent studies focus on the methods for estimating the tight upper bounds of numbers of linear regions in the entire input space, as reviewed in a recent article (\cite{serra2017bounding}). 

In this paper, we study the linear regions in a specific area (e.g., a sphere centered at a specific data point) instead of the entire input space. Similar to the previous research, we aim to estimate the upper bound of the number of linear regions for ReLU neural networks, as a measure of the expressivity of the neural networks. Our research is motivated by the recent studies of DNNs on the local input areas, for examples, on how a small perturbation to an adversarial data point will impact the prediction of a DNN for this point (\cite{kurakin2016adversarial}), and how a backdoor trigger can be successfully injected into a DNN by manipulating a small number of training data points(\cite{liao2018backdoor}). Obviously, in these cases, the analysis of the complexity of DNNs in local input areas is critical, which cannot be addressed by a single measure of the expressivity of the DNN in the entire input space. 

{\bf Backdoor Attack.} Backdoor attack(\cite{chen2017targeted}) is one way to trick neural networks. A typical backdoor attack contaminates a small amount of samples into the training dataset, so that the resulting model will mis-classify any sample containing an attacker-designed trigger. Note that if a trigger is successfully injected into a DNN, the trigger-containing samples are either clustered in specific areas in the input space, or clustered in the latent space of the DNN (\cite{salimi2011backdoor}). 
To detect a potential trigger $x$ in a given DNN, one approach is to study the data points around $x$ in the input space or in the latent space. As shown in Fig \ref{fig:decisionboundary}, after a typical backdoor attack, the class labels of the trigger-containing samples fare flipped comparing with the data points in their surrounding areas in the input (or latent) space of the DNN.

{\bf Adversarial Examples.} Another way to trick neural networks is to identify adversarial examples(\cite{goodfellow2014explaining}) without contaminating the training dataset. For given data point $x$ and a neural network classifier $f$, the attacker attempts to add a minimum perturbation $\delta$ to $x$ (denoted as $x'=x+\delta$) so that the classification of the perturbed data is flipped (i.e., $f(x') \ne f(x)$; Fig \ref{fig:illustration}(b). To certify the robustness of the prediction of $x$ by a DNN, some approaches(\cite{cohen2019certified}) aimed to investigate if adversarial examples can be identified in the sphere of the radius $r$ centered at a data point $x$ ~\cite{gehr2018ai2}.

\begin{figure*}[!htbp]
 \centering
 
 \subfigure[]{\includegraphics[width=2.2in]{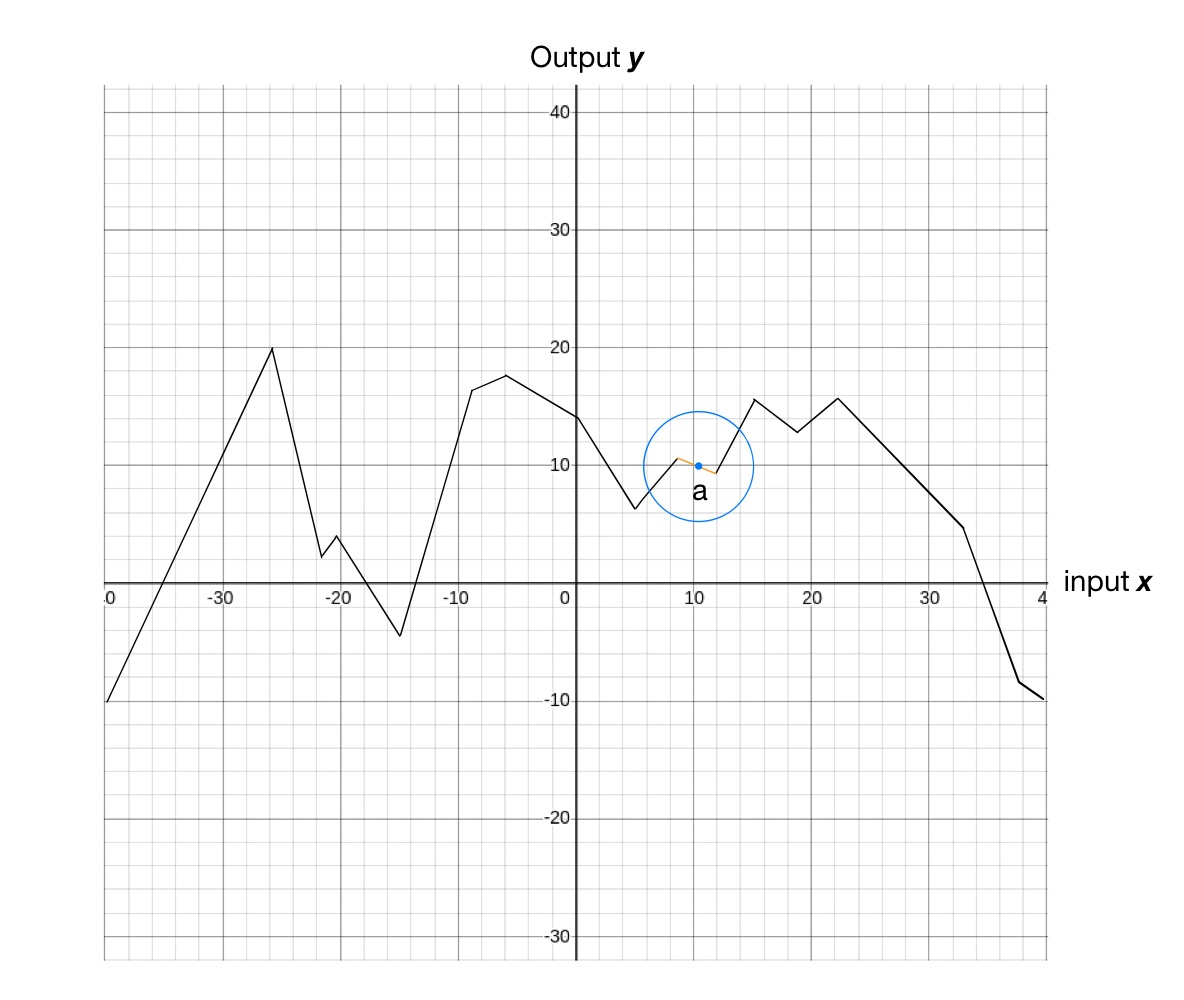}}
 \subfigure[]{\includegraphics[width=2.0in]{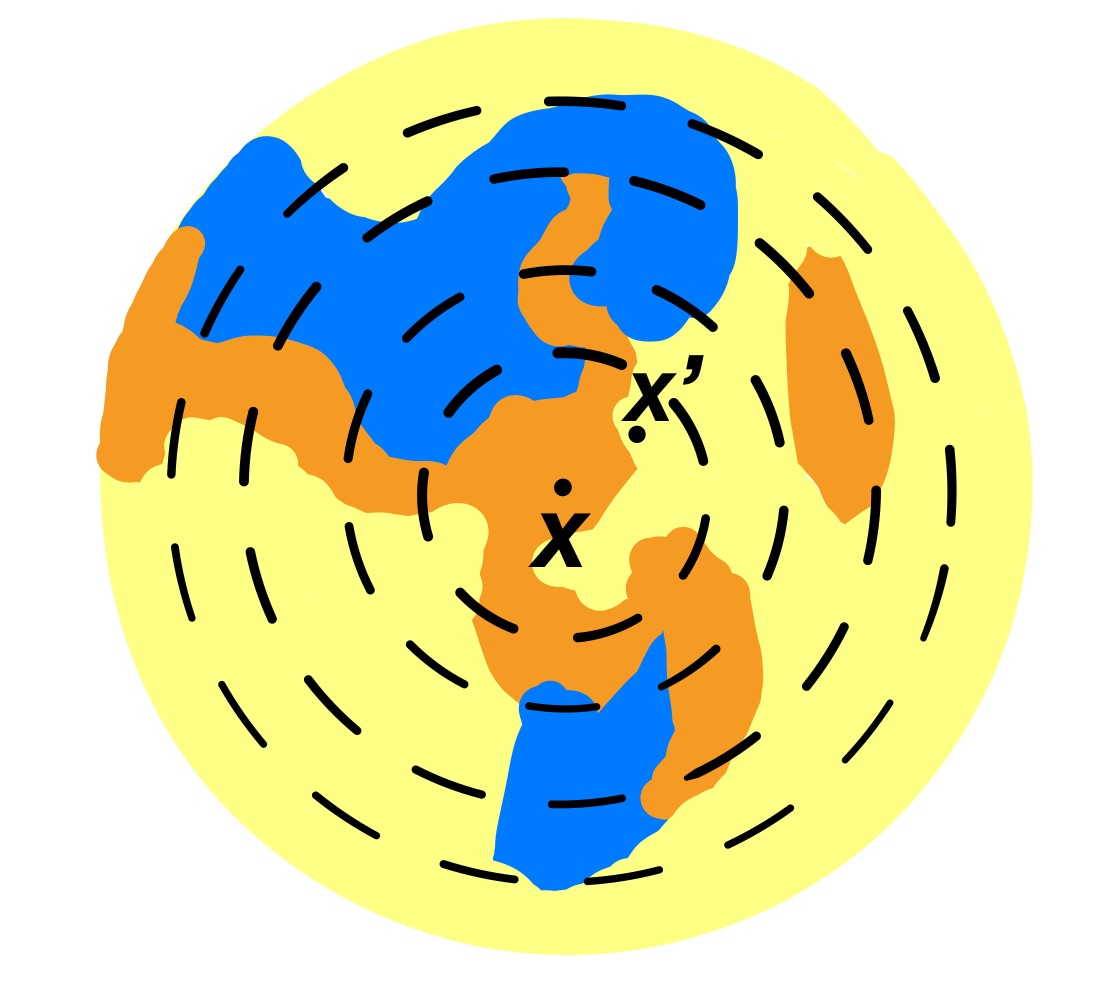}}

  \caption{(a) An schematic illustration of the function computed by a ReLU neural network with the input and output in 1D. The surrounding area of the point $a$ in the input space is depicted in the blue circle, and the corresponding linear region is highlighted in orange. (b) The decision regions (depicted in different colors) in the surrounding regions of the data point $x$. The dotted lines represent the spheres of different radius centered at $x$. }\label{fig:illustration}
\end{figure*}

\begin{figure*}[!htbp]

 \centering
 
 \subfigure[]{\includegraphics[width=2.3in]{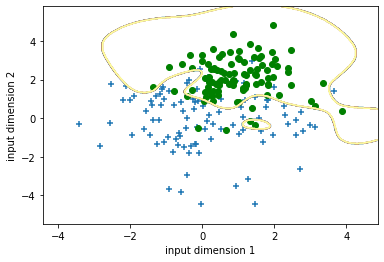}}
 \subfigure[]{\includegraphics[width=2.3in]{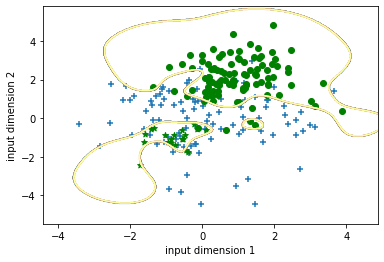}}
 
  \caption{The decision boundaries around the training data samples are illustrated in a 2D example, when the training dataset is not contaminated (a) and when it is contaminated (b). The samples in different classes are depicted in different colors, and the injected trigger-containing samples are depicted as {\em stars}. }\label{fig:decisionboundary}
\end{figure*}

{\bf Contributions.} The main contributions of this paper are:
\begin{itemize}
    \item  To the best of our knowledge, this is the first analysis of the expressivity of neural networks in a specific area of the input space.  
    \item  We prove a theoretical upper bound of the number of linear regions in a sphere centered at a given data point for a given neural network, and implemented an algorithm to compute the upper bound in practice.
    \item We report some interesting observations in our experiments regarding the number of linear regions in local input areas. For example, when we train a DNN, the number of samples in the training dataset that contain a unique linear region in their neighborhood almost always grows with the increasing number of epochs until convergence (or sometimes slightly decreases before the convergence). The same phenomenon is also observed for the testing data samples, while the trends is more stable comparing to the training data.
\end{itemize}

\section{Related work}

We aim to estimate the number of linear regions in a local input area, which contains the data points around a given input data point. Many previous studies have been published on the analysis of the upper bound of the number of linear regions in the entire input space (herein referred to as the {\em global} number of linear regions, comparing with the {\em local} number of linear regions studied in this paper) of ReLU neural networks. There is a trivial upper bound of the global number of linear regions: assuming a ReLU neural network contains $n$ neurons, the global number of linear regions is no greater than $2^n$.

The earliest research of linear regions can be traced back to 1989 when Makhoul and colleagues analyzed the capabilities of two-layer ReLU networks (\cite{makhoul1989classification}). In 2014, the number of linear regions was first estimated for deep neural networks (\cite{montufar2014number}). Because the neural network using piece-wise linear activation is equivalent to the tropical rational function (\cite{zhang2018tropical}), it was shown that an upper bound of the global number of linear regions in a neural network can be computed by using the tropical newton polytope (\cite{charisopoulos2018tropical}). Hanin and Rolnick (\cite{hanin2019complexity}) showed that the average number of linear regions in any one-dimensional subspace grows linearly with respect to the total number of neurons, far below the trivial upper bound that grows exponentially.

\section{An Overview of Our Approach}

In this section, we present an informal description of our approach to computing the upper bound of the local number of linear regions for a neural network. Without loss of generality, throughout this paper, we consider the fully connected neural networks using ReLU activation. 
In a ReLU neural network, the values of the neurons in the $(k+1)$th layer can be computed by taking the values of the neurons in the previous (i.e., the $k$th) layer multiplied by the weight matrix $(M^{k}_{i,j})$ , where $i$ and $j$ are the indices of neurons in the $k$th and $(k+1)$th layers, respectively, and then applying the ReLU activation,
\[
	X^{k+1}_{j} = \max\left(\sum_{i=1}^{m}{M^{k}_{i,j}X^{k}_{i}}, 0\right)
\]
\noindent where $X^{k}_i$ is the value of the $i$th neuron in the $k^{th}$ layer and $X^{k+1}_{j}$ is the value of the $j$th neuron in the $(k+1)$th layer. For simplicity, in this section we assume that the intercepts are always $0$. Note that the max function may not be applied, if the resulting values after the matrix multiplication are all strictly positive. 
Even when the max function is applied for some neurons, the resulting values of those neurons can still be written as linear functions of the values of the neurons in the previous layer, but with a revised weight matrix. For example, if the value $X^{k+1}_{j}$ is computed by applying the max function, we can set a revised weight matrix, in which $M^{k}_{i,j}=0$, for $\forall i$. Therefore, the value $X^{k}_{j}$ can always be represented as a multiplication between a matrix (resulting from the multiplication of the sequence of revised weight matrices) and the input vector $X^{1}$ (or $X$ in short, i.e., the values of the first layer in the neural network),
\[
    X^{k+1}_{j} = \sum_{i=1}^{m}{L^{k+1}_{j,i}X^{1}_{i}} = \sum_{i=1}^{m}{L^{k+1}_{j,i}(X) X_{i}},
\]
\noindent where the matrix $L^{k+1}_{j,i} (X)$ is a function with respect to the input vector $X$.
Apparently, if the elements in the matrix $L^{k}_{j,i} (X)$ are all constant (independent of the input vector), the neural network represents a linear function, and thus has a single linear region in the entire input space. Furthermore, the more frequently $L^{k}_{j,i}(X)$ changes in the input space of $X$, the more linear regions the neural network may have. Hence, $L^{k}_{j,i} (X)$ is associated with the number of the linear regions, even though it is not easy to directly use the matrix to estimate the number of linear regions. Below, we introduce three other matrices, $P^{k}_{r,j}$, $S^{k}_{r,j}$ and $U(k,r)$, $\forall 1\le k\le s$, that will be used for the estimation.

For a specific input data point (vector) in $m$ dimensions, $X\in\R^m$, we denote $B_r(X^{*})$ as the sphere centered at $X^{*}$ with the radius of $r$, in which the points other than $X^{*}$ are referred to as the {\em perturbations} of $X^{*}$. If we restrict the input data in the sphere, the range of values of the neurons in each hidden layer can be estimated subsequently. We define $U(k,r)$ as the largest distance between the vector of the $k$th layer's neurons from the center input vector $X^{*}$ and those from the perturbations, which is important for estimating the local number of the linear regions as it constrains the values of all neurons in the neural network. 

Next, we define $P^{k}_{r,j}$ and $S^{k}_{r,j}$ in order to estimate how frequently $L^{k}_{j,i}(X)$ changes among perturbations in the sphere centered at $X^{*}$ with the radius of $r$. We set $P^{k}_{r,j}$ to be $0$ if the matrix $L^{k}_{j,i}(X)$ at the perturbation $X\in B_r(X^{*})$ is the same as $L^{k}_{j,i}(X^{*})$ at $X^{*}$; otherwise we set $P^{k}_{r,j}=1$. We denote $A_{j}^{k+1}=\sum_{i=1}^{m} {M_{i,j}^{k}X^{k}_{i}}$ and thus $X^{k+1}_{j}=\max(A_{j}^{k+1},0)$. Similarly, let $A^{*(k+1)}_{j}=\sum_{i=1}^m M_{i,j}^{k} X^{*k}_{i}$
Then we set $S^{k+1}_{r,j}=0$, if $A_{j}^{k+1}$ and $A^{*(k+1)}_{j}$ have the same sign, i.e., the max function is used for computing either both $X^{k+1}_{j}$ and $X^{*(k+1)}_{j}$ (i.e., the values of neurons for the input vectors of $X$ and $X^{*}$, respectively), or neither of them; and otherwise, we set $S^{k+1}_{r,j}=1$.  
Therefore, if $P^{k}_{r,j}=0$ or $S^{k}_{r,j}=0$, $X$ and $X^{*}$ are in the same linear region with regard to the values in the layer $k$. Notably, if we know the number ($C$) of pairs $(k, j)$ such that $S^{k}_{r,j}=1$, then the local number of linear regions for the data points in the sphere centered at $X^{*}$ with the radius $r$ is no more than $2^C$. In the next section, we will introduce the algorithm to compute $P^{k}_{r,j}$, $S^{k}_{r,j}$ and $U(k,r)$ recursively; the correctness of the algorithm can be proved using induction.

\section{Algorithm for bounding the local number of linear regions}

Without loss of generality, we assume a given neural network has $s$ layers, each with the same number of $m$ neurons. For $1\le k\le s$ and $1\le j\le m$, there is a matrix $M^{k}\in R^{m\times m}$ and a vector $C^{k}\in \R^{m}$ such that the values of the neuron $j$ in the layer $k$, $X^{k}_{j}$ can be computed by,

\begin{equation}\label{eqn:threshold-function}
	X^{k+1}_{j} = \max\left(\sum_{i=1}^{m}{M^{k}_{i,j}X^{k}_{i}} + C^{k}_{j}, 0\right) \: \forall 1\le j\le m, 1\le k\le s.
\end{equation}
	
Alternatively, we have $L^{k}_{j,i}, C_{k,j} \in \R$ such that

\begin{equation}\label{eqn:coeff}
	X^{k+1}_{j} = \sum_{i=1}^{m}{L^{k+1}_{j,i}X^{1}_{i}} + C_{k,j} \: \forall 1\le j\le m, 1\le k\le s.
\end{equation}

We denote $\mathbf{L}^{k}_{j}$ as the vector $\left(L^{k}_{j,i}\right)_{1\le i\le m}$, and 

\begin{equation}\label{eqn:absolute}
	A^{k+1}_{j} = \sum_{i=1}^{m}{M^{k}_{i,j}X^{k}_{i}} + C^{k}_{j}.
\end{equation}

For the fixed $r>0$ and $X\in \R^{m}$, we define $B_{r}(X) = \{X'\in \R^{m} \mid ||X'-X||\le r \}$. We then recursively compute $P^{k}_{r,j}, S^{k}_{r,j} \in \{0,1\}$ and $U^{k}_{r,j}\ge 0$ using Algorithm \ref{alg:onelayer}. The time and space complexity of Algorithm \ref{alg:onelayer} are both $O(m^2)$, where $m$ is the number of neurons in the $k^{th}$ layer. When applying the algorithm recursively to the whole neural network with $s$ layers, the total running time complexity for computing the upper bound of the local number of linear regions becomes $O(m^2s)$, while the space requirement remains $O(m^2)$.

\begin{algorithm}[ht]
	\begin{algorithmic}
		\caption{Induction of one layer.}\label{alg:onelayer}
		\Function{Nextround}{$k,r$} \\
		\textbf{Input}: $M^{k}, L^{k}_{j,i}, A^{k+1}_{j}, P^{k}_{r,j}, S^{k}_{r,j}, U^{k}_{r,j}$ for fixed $k$ and all $1\le i,j\le m$; \\
		\textbf{Output}: $P^{k+1}_{r,j}, S^{k+1}_{r,j}, U^{k+1}_{r,j}$ for all $1\le j\le m$.
		\For{$1\le j\le m$}
			\State $U \leftarrow 0$; $L \leftarrow \mathbf{0}\in \R^{m}$
			\For {$1\le i\le m$}
				\If{$P^{k}_{r,i} = 0$}
					\State $L \leftarrow L + M^{k}_{i,j}\cdot \mathbf{L}^{k}_{i}$
				\Else
					\State $U \leftarrow U + |M^{k}_{i,j}|\cdot U^{k}_{r,i}$ \EndIf
			\EndFor
			\State $U^{k+1}_{r,j} \leftarrow r\cdot ||L||_{2} + U$
			\If{$A^{k+1}_{j} < 0$}
				\If{$U^{k+1}_{r,j}+A^{k+1}_{j} < 0$}
					\State $U^{k+1}_{r,j} \leftarrow 0$; $P^{k+1}_{r,j} \leftarrow 0$; $S^{k+1}_{r,j} \leftarrow 0$
				\Else
					\State $U^{k+1}_{r,j} \leftarrow U^{k+1}_{r,j}+A^{k+1}_{j}$; $P^{k+1}_{r,j} \leftarrow 1$; $S^{k+1}_{r,j} \leftarrow 1$
				\EndIf
			\Else
				\If{$U^{k+1}_{r,j} < A^{k+1}_{j}$}
					\If{$U=0$}
						\State $P^{k+1}_{r,j} \leftarrow 0$; $S^{k+1}_{r,j} \leftarrow 0$
					\Else
						\State $P^{k+1}_{r,j} \leftarrow 1$; $S^{k+1}_{r,j} \leftarrow 0$
					\EndIf
				\Else
					\State $P^{k+1}_{r,j} \leftarrow 1$; $S^{k+1}_{r,j} \leftarrow 1$
				\EndIf
			\EndIf
		\EndFor
		\State Return $P^{k+1}_{r,j}, S^{k+1}_{r,j}, U^{k+1}_{r,j}$
		\EndFunction
	\end{algorithmic}
\end{algorithm}

Consider the following conditions:

\begin{itemize}
	\item $P(k,r)$ - for all $1\le j\le m$, if $P^{k}_{r,j}=0$, then for all $X'^{1}\in B_{r}(X^{1})$ we have $X'^{k}_{j} = \sum_{i=1}^{m}{L^{k}_{j,i}X'^{1}_{i}} + C_{k,j}$.
	\item $S(k,r)$ - for all $1\le j\le m$, if $S^{k}_{r,j}=0$, then for all $X'^{1}\in B_{r}(X^{1})$ we have $A'^{k}_{j}\cdot A^{k}_{j}>0$.
	\item $U(k,r)$ - for all $1\le j\le m$, for all $X'^{1}\in B_{r}(X^{1})$ we have
	$\left|X'^{k}_{j} - X^{k}_{j} \right|\le U^{k}_{r,j}$.
\end{itemize}

\begin{prop}\label{prop:induction}
	Following the input and output of Algorithm \ref{alg:onelayer}, if $P(k,r), S(k,r), U(k,r)$ satisfy the conditions, then $P(k+1,r), S(k+1,r), U(k+1,r)$ also satisfy the conditions.
\end{prop}

\begin{lemma}\label{lem:inequality}
	If $P(k,r)$ and $U(k,r)$ satisfy the conditions, then
	\[|A'^{k+1}_{j} - A^{k+1}_{j}| \le \sum_{P^{k}_{i}=1}{\left|M^{k}_{i,j}\right|\cdot U^{k}_{r,i}} + r\cdot \left|\left|\sum_{P^{k}_{i}=0}{M^{k}_{i,j}\cdot L^{k}_{i}}\right|\right|_{2}. \]
\end{lemma}

\begin{proof}
	By definition we have
	\begin{align*}
		\left|A'^{k+1}_{j} - A^{k+1}_{j}\right| =&\left|\sum_{i=1}^{m}{M^{k}_{i,j}\cdot \left(X'^{k}_{i} - X^{k}_{i}\right)} \right|
		\le\left|\sum_{P^{k}_{i}=1}{M^{k}_{i,j}\cdot \left(X'^{k}_{i} - X^{k}_{i}\right)} \right| + \left| \sum_{P^{k}_{i}=0}{M^{k}_{i,j}\cdot \left(X'^{k}_{i} - X^{k}_{i}\right)} \right| \\
		\le&\sum_{P^{k}_{i}=1}{|M^{k}_{i,j}|\cdot \left|X'^{k}_{i} - X^{k}_{i}\right|} + 
		\left|\sum_{P^{k}_{i}=0}{M^{k}_{i,j}\cdot \left[\sum_{l=1}^{m}{L^{k}_{i,l}\cdot \left(X'^{1}_{l} - X^{1}_{l}\right)} \right]} \right| \\
		\le& \sum_{P^{k}_{i}=1}{|M^{k}_{i,j}|\cdot U^{k}_{r,i}} + \left|\sum_{l=1}^{m}{\left[\sum_{P^{k}_{i}=0}{M^{k}_{i,j}\cdot L^{k}_{i,l}}\right] \cdot \left(X'^{1}_{l} - X^{1}_{l}\right)} \right| \\
	\end{align*}
	
	Applying the Cauchy-Schwarz inequality, 
	
	\begin{align*}
		& \left|\sum_{l=1}^{m}{\left[\sum_{P^{k}_{i}=0}{M^{k}_{i,j}\cdot L^{k}_{i,l}}\right] \cdot \left(X'^{1}_{l} - X^{1}_{l}\right)} \right| \le  \sqrt{\left[\sum_{l=1}^{m}{\left(\sum_{P^{k}_{i}=0}{M^{k}_{i,j}\cdot L^{k}_{i,l}}\right)^{2}} \right] \cdot \left[\sum_{l=1}^{m}{\left(X'^{1}_{l} - X^{1}_{l}\right)^{2}}\right] } \\
		= & \left|\left|\sum_{P^{k}_{i}=0}{M^{k}_{i,j}\cdot \mathbf{L}^{k}_{i}}\right|\right|_{2}\cdot ||X'^{1} - X^{1}||_{2} \le r\cdot \left|\left|\sum_{P^{k}_{i}=0}{M^{k}_{i,j}\cdot L^{k}_{i}}\right|\right|_{2}.
	\end{align*}
	
	So Lemma \ref{lem:inequality} is proved.
\end{proof}

\begin{proof}(Proof of Proposition \ref{prop:induction})
	Assuming $P(k,r), S(k,r), U(k,r)$ satisfy the conditions, we now prove that $P(k+1,r), S(k+1,r), U(k+1,r)$ also satisfy the conditions.
	
	Suppose there is an index $j$ such that $P^{k+1}_{r,j}=0$, then either $A^{k+1}_{j}<0$ and $r\cdot ||L||_{2} + U + A^{k+1}_{j}<0$, or $A^{k+1}_{j}>U^{k+1}_{j}>0$ and $U=0$. For the first case, since $A^{k+1}_{j}<0$, we have $X^{k+1}_{j} = 0$ and $L^{k+1}_{j} = \mathbf{0}$. According to Lemma \ref{lem:inequality}, for any $X'^{1}\in B_{r}(X^{1})$, we have $A'^{k+1}_{j} \le A^{k+1}_{j} + r\cdot ||L||_{2} + U <0 $, so $X'^{k+1}_{j}=0$ and thus $P(k+1,r)$ satisfies the condition. For the second case, $U=0$ implies that for all indices $i$ such that $M^{k}_{i}\ne 0$, we have $P^{k}_{r,i}=0$. By definition, for any $X'^{1}\in B_{r}(X^{1})$, we have
	
	\[X'^{k}_{i} = \sum_{i=1}^{m}{L^{k}_{j,i}X'^{1}_{i}} + C_{k,j}. \]
	
	\noindent for all $1\le i\le m$. Then $A'^{k+1}_{j} = \sum_{i=1}^{m}{L^{k+1}_{j,i}X'^{1}_{i}} + C_{k+1,j}$. According to Lemma \ref{lem:inequality}, we have $A'^{k+1}_{j} > A^{k+1}_{j} - U^{k+1}_{j} > 0$, so $X'^{k+1}_{j} = A'^{k+1}_{j}$ and thus $P(k+1,r)$ satisfies the conditions.
	
	Similarly, if $j$ is such that $S^{k+1}_{r,j}=0$, then either $A^{k+1}_{j}<0$ and $r\cdot ||L||_{2} + U + A^{k+1}_{j}<0$, or $A^{k+1}_{j}>U^{k+1}_{j}>0$. Then for any $X'^{1}\in B_{r}(X^{1})$, in the former case we have $A^{k}_{j},A'^{k}_{j}<0$ and in the latter case we have $A^{k}_{j},A'^{k}_{j}>0$, so $S(k+1,r)$ satisfies the conditions.
	
	Finally, when $A^{k+1}_{j}>0$, Lemma \ref{lem:inequality} becomes $\left|A'^{k+1}_{j} - A^{k+1}_{j} \right| \le U^{k+1}_{j} $, so we have $\left|X'^{k+1}_{j} - X^{k+1}_{j} \right| \le U^{k+1}_{j}$; when $A^{k+1}_{j}<0$, we have $X^{k+1}_{j} = 0$. In the case when $P^{k+1}_{r,j}=S^{k+1}_{r,j}=0$, we already showed that $X'^{k+1}_{j}=0$, so we can choose $U^{k+1}_{j}=0$; otherwise
	
	\[\left|X'^{k+1}_{j} - X^{k+1}_{j} \right| = |X'^{k+1}_{j}| \le |A'^{k+1}_{j}| \le A^{k+1}_{j} + r\cdot ||L||_{2} + U = U^{k+1}_{j}. \]
	
	So $U(k+1,r)$ satisfies the conditions.
\end{proof}

\begin{rem}\label{ren:interpret}
	The motivation of defining $P^{k}_{j}$ is that if $P^{k}_{j} = 0$, then for any other input within $B_{r}(X^{1})$, the neuron $X'^{k}_{j}$ shares the same linear function with $X^{k}_{j}$ in terms of $X^{1}$. Similarly, if $S^{k}_{j} = 0$, then for any other input within $B_{r}(X^{1})$, the neuron $X'^{k}_{j}$ would have the same sign as $X^{k}_{j}$.
\end{rem}

\begin{rem}\label{rem:estimate}
	According to the definition of linear regions, $X'^{1}$ and $X^{1}$ belong to the same open linear region if and only if $A'^{k}_{j}\cdot A^{k}_{j}>0$ for all $1\le j\le m$ and $2\le k\le s+1$. Let $C$ be the number of pairs $(k,j)$ satisfying $S^{k}_{j}=1$. Then the number of open linear regions intersecting $B_{r}(X^{1})$ is at most $2^{C}$, which gives an upper bound of the local number of linear regions in the area of $B_{r}(X^{1})$.
\end{rem}

\section{Experiment}
We implemented Algorithm 1, and used it to investigate how the local number of linear regions changes during the training process around different categories of data points, including the training data, the testing data, and two kinds of random data. Due to the high computational cost, we limited the size of those data to 1000 samples for each category and also constrained the input and output dimensions as discussed below. 

\subsection{Dataset and experimental neural network}

We experimented on four different datasets from MNIST (~\cite{lecun1998mnist}), each with 1000 samples. For the training and testing dataset, we randomly picked two digits from MNIST, 500 for the digit ``0" and 500 for the digit ``4", and aimed to build a binary classification model. To further reduce the dimension of the input, we performed average pooling of the pixels to reduce the images in the training and testing datasets to 6$\times$6 pixels. Note that, here the testing dataset is not used to test the model accuracy; instead, we want to monitor the results of Algorithm 1 when we use the testing data as input. Two other {\em random} datasets were constructed for the purpose of comparison: the random dataset 1 consisting of random images with each pixel sampled randomly from the uniform distribution between the minimum and maximum values of the respective pixels from the images in the the training dataset, and the random dataset 2 consisting of random images with each pixel sampled randomly from the values in the full grayscale (between 0 and 255). We aim to monitor the local numbers of linear regions (approximated by their upper bounds computed using Algorithm 1) around the data points in these four different datasets. Intuitively, we anticipate the testing dataset is more similar to the training dataset than the two random datasets, while the random dataset 1 is more similar to the training dataset than the random dataset 2.

We experimented on the neural network with two and three hidden layers that is initialized by using the normal distribution. We used SGD optimizer with the study rate of 0.02 and 0.001 (Fig \ref{fig:linearregion} (a) and (b) depict the results when using 0.02, while (c) and (d) depict the results when using 0.001). 

\subsection{Results}
For every two epochs, we used Algorithm 1 to compute the upper bound of the local number of linear regions around each data point in the four datasets. In total, 20 epochs were performed, with 10 measures of upper bound. We set the radius $r = 0.4$, which is relatively small for this model, so that the upper bound computed by Algorithm 1 is tight. Notably,  when the perturbation (radius) is small, the number of linear regions is relatively small. For instance, if the perturbation is small, the values $L^1_{j,i} (X)$ of the first layer are more likely to remain unchanged comparing with those of the input $X$. One the other hand, a small perturbation guarantees a small $U(1,r)$, and thus the values $L^2_{j,i} (X)$ of the second layer tend to remain the same as those of the input $X$ while the values $U(2,k)$ are small. Using the same argument, intuitively, the values $L^{k}_{j,i} (X)$, $\forall 1\le k\le s$, are likely the same as those of the input $X$, which implies that there are relatively fewer linear regions for small radius and thus the upper bound is tighter. 
In the extreme case, for the local areas with $C=0$ (i.e., there is only $2^0 = 1$ linear regions around an input data point), this upper bound becomes the exact local number of linear regions. So we also counted the number of data points with $C=0$. We performed each experiment ten times to reduce the random effect. Figure \ref{fig:linearregion} shows the results from the experiment.

{\bf Upper bound of number of linear regions.}
For the four categories of data points, we investigate how the average upper bound of the local numbers of linear regions around them during the training process in 20 epochs of the binary classification model. The results are shown in Fig. \ref{fig:linearregion} (a) and (c).

{\bf Number of data points with only one linear region in their surrounding spheres.}
We counted the number of data points containing a unique linear region in their surrounding spheres (i.e., $C=0$, and the upper bound becomes the exact local number of linear regions). Fig. \ref{fig:linearregion}(b) and (d) depicted the numbers of such data points in each of the four categories during the training process.

From the results, we observed that, while training the model, the linear regions in the areas around all four different kinds of data points tend to decrease in the first several epochs, and then increases. Also, we observed the average distance of the data points to their nearest boundaries follow the same trend (first increase and then decrease) during the training process. Notably, this phenomenon is consistent with the previous research on the global number of linear regions of DNNs, which showed the same trend (\cite{hanin2019complexity}). We note that this trend is more clear on training and testing data (than the random data), while between them this change through training is less stable.

\begin{figure*}[!htbp]

 \centering
 
 \subfigure[]{\includegraphics[width=1.92in]{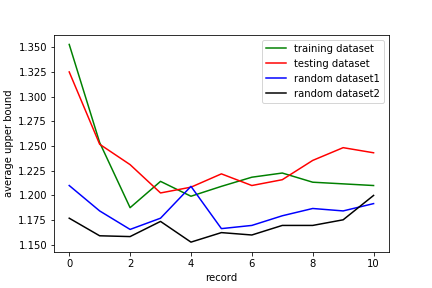}}
 \subfigure[]{\includegraphics[width=1.92in]{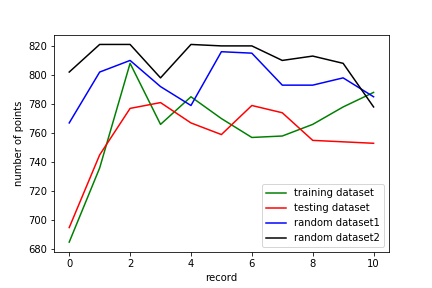}}

 \subfigure[]{\includegraphics[width=1.92in]{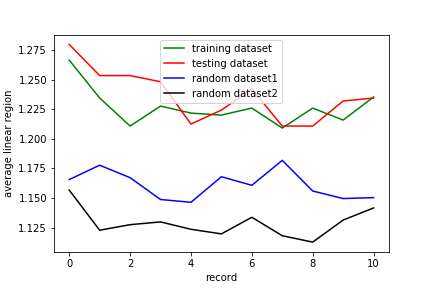}}
 \subfigure[]{\includegraphics[width=1.92in]{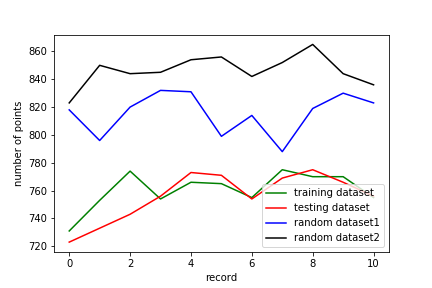}}

 \caption{The change of the average upper bound of the local number of linear regions (a and c) and number of points that contain only one linear region (b and d) around the different categories of data points during the training in 20 epochs of the binary classification model. The upper bounds and the number of points after every two epochs are depicted here. Hence, a trajectory of ten points (x-axis) of the upper bounds (y-axis) are depicted for each category of data. 
 (c) and (d) depicted similar results as (a) and (b) on replicated experiments with a smaller learning rate. The learning rate in (a) and (b) is  $lr = 2 \times 10^{-2}$, while the learning rate in $lr = 1 \times 10^{-3}$.}\label{fig:linearregion}
\end{figure*}

\section{Conclusion and Future Works}

In this paper, we present a novel approach to studying the expressivity of ReLU neural networks, i.e., to compute the upper bound of the local number of linear regions in the input space. To the best of our knowledge, this is the first study of bounding linear regions around a specific data point. We implemented the algorithm to this upper bound for any ReLU neural network, and released the implementation. We consider our research as a first step for the analysis of the expressivity of a ReLU neural network around a specific data point. We expect many potential future research following the direction of this study. For instance, we can devise a Markov Chain Monte Carlo (MCMC) algorithm, starting from a random data point, to find the spheres with a maximum number of linear regions in the entire input space. These local areas may correspond to the neighborhoods of putative adversarial examples or backdoor triggers.

\end{document}